\documentclass{jams-l}

\usepackage[utf8]{inputenc}
\usepackage[margin=1.15in]{geometry}

\newcommand{\R}{\mathbb{R}}
\newcommand{\E}{\mathop\mathbb{E}}

\usepackage[round, numbers]{natbib}
\usepackage{enumerate}
\usepackage{graphicx}
\usepackage{xcolor}

\usepackage{amssymb}
\newcommand{\cL}{\mathcal{L}}
\usepackage{hyperref}       
\newcommand{\norm}[1]{\left\|#1\right\|}
\usepackage[capitalise]{cleveref}

\newtheorem{theorem}{Theorem}
\newtheorem*{theorem*}{Theorem}
\newtheorem*{claim*}{Claim}
\newtheorem{lemma}[theorem]{Lemma}
\newtheorem{corollary}{Corollary}

\newtheorem{claim}[theorem]{Claim}
\newtheorem{prop}[theorem]{Proposition}

\newtheorem*{lemma*}{Lemma}
\newtheorem*{obs*}{Observation}

\newcommand{\eps}{\varepsilon}
\newcommand{\g}{g}
\newcommand{\A}{a}
\newcommand{\invrad}{\rad_\cK^{-1}}

\newcommand{\cB}{\mathcal{B}}

\newcommand{\D}{\mathcal{D}}
\newcommand{\argm}{\mathsf{argmin}}

\newcommand{\N}{\mathbb{N}}

\newcommand{\cK}{\mathcal{K}}
\renewcommand{\cL}{\mathcal{L}}
\newcommand{\rad}{\mathsf{Rad}}
\newcommand{\rep}{\mathsf{Rep}}
\newcommand{\cov}{\mathsf{Cov}}
\newcommand{\cD}{\mathcal{D}}

\newcommand{\ip}[2]{\langle #1,#2 \rangle}

\numberwithin{equation}{section}

\begin{document}

\title[The sample complexity of ERMs in SCO]{The sample complexity of ERMs \\ 
in stochastic convex optimization}



\newcommand{\ignore}[1]{}
\author[Carmon]{Daniel Carmon}
\address{Department of Mathematics, Technion-IIT}
\email{daniel.carmon91@gmail.com}

\author[Livni]{Roi Livni}
\address{Department of Electrical Engineering, Tel Aviv University}
\email{rlivni@tauex.tau.ac.il}
\thanks{R.L. is a recipient of a Google research scholar award and would like to acknowledge his thanks. The research was supported, in part, by an ISF Grant (2188 $\backslash$ 20), as well as supported by the ERC grant (GENERALIZATION, 10139692).}

\author[Yehudayoff]{Amir Yehudayoff}
\address{Department of Computer Science, University of Copenhagen
and Department of Mathematics, Technion-IIT}
\email{amir.yehudayoff@gmail.com}
\thanks{A.Y.\
is part of Basic Algorithms Research Copenhagen supported by the VILLUM Foundation grant 16582,
and also thanks the Pioneer Centre for AI, DNRF grant number P1.}

\begin{abstract}

    Stochastic convex optimization is one of the most well-studied models for learning in modern machine learning. Nevertheless, a central fundamental question in this setup remained unresolved: \begin{quote}
{\em how many data points must be observed so that any 
    empirical risk minimizer (ERM) shows good performance on the true population?}    
    \end{quote}
    This question was proposed by Feldman 
    who proved that $\Omega(\frac{d}{\epsilon}+\frac{1}{\epsilon^2})$ data points are necessary (where~$d$ is the dimension and $\eps>0$ is the accuracy parameter). Proving an
    $\omega(\frac{d}{\epsilon}+\frac{1}{\epsilon^2})$ lower bound was left as an open problem. 
    In this work we show that in fact $\tilde O(\frac{d}{\epsilon}+\frac{1}{\epsilon^2})$ data points are also
    sufficient. 
This settles the question and yields a new separation between ERMs and uniform convergence.


    This sample complexity holds for the classical setup of learning bounded convex Lipschitz functions over the Euclidean unit ball. We further generalize the result and show that a similar upper bound holds for all symmetric convex bodies.
    The general bound is composed of two terms: (i) a 
    term of the form $\tilde O(\frac{d}{\epsilon})$ with an inverse-linear dependence on the accuracy parameter, and (ii) a term that depends on the statistical complexity of the class of \emph{linear} functions (captured by the Rademacher complexity).
    The proof builds a mechanism for controlling the behavior of stochastic convex optimization problems. 
    
\end{abstract}

\maketitle

\section{Introduction}
Stochastic convex optimization (SCO) is a benchmark framework that is widely used for studying stochastic optimization algorithms such as gradient descent and its variants. This is often justified by the simplicity of the framework and the possibility of a rigorous analysis that can hint at the pros and cons of various optimization techniques in practical setups such as machine learning.
It has also been studied in the optimization literature under the name
sample average approximation (SAA);
see for example~\cite{shapiro2003monte,nemirovski2009robust} and references within.

Stochastic convex optimization is particularly useful for understanding the interaction between optimization and generalization in complex scenarios.
The works of \citet*{shalev2009stochastic}, and subsequently \citet{feldman2016generalization}, demonstrated 
how the choice of an algorithm is crucial not only for optimization reasons, but also for generalization reasons. Namely, to avoid overfitting without careful algorithmic choices, one must use
dimension-dependent sample size. On the other hand, with the correct algorithm, one can avoid overfitting with far less data points. For this reason, SCO became a prototypical model for researching \emph{over-paramterization}~~\cite{dauber2020can, amir2021sgd, koren2022benign, sekhari2021sgd}.
These works aim to understand how can algorithms avoid overfitting, even when the number of data points is significantly smaller than the number of free parameters.

In more detail, a classical paradigm for learning is to draw $n$ i.i.d.\ data points $z_1,\ldots, z_n$ from an unknown distribution and to optimize the \emph{empirical risk} defined as
\[ \hat{F}(x) = \frac{1}{n} \sum^n_{i=1} f(x,z_i),\]
where $f(x,z)$ is some loss function that measures the performance of parameter $x$ on the data point~$z$. When $n$ is sufficiently large, a minimizer of $\hat{F}$ should also demonstrate good performance on the population loss $F(x)= \E_{z\sim D}[f(x,z)]$. 
This is exactly the setting of SCO,
expect that SCO makes the additional assumption that the loss functions
are convex (more details are provided below).

It is a well-known fact that in the Probably Approximately Correct framework \cite{vapnik2015uniform, valiant1984theory}, all optimization algorithms share the same statistical rate~\cite{vapnik2015uniform, blumer1989learnability}. 
Namely, 
the statistical rate of {\em all} learning algorithms is the same.
And understanding the rate of \emph{any} learning algorithm is equivalent, then, to understanding the rate of any \emph{specific} ERM.
In stochastic convex optimization, however, not all algorithms are equal.
While there are learning algorithms that perform well with only $n=O\ (\frac{1}{\eps^2} )$ data points \cite{shalev2009stochastic,bousquet2002stability}, others must observe at least $\Omega (\frac{d}{\eps})$ points to guarantee good performance where $d$ is the dimension of the parameter space $x\in \R^d$~\cite{feldman2016generalization}. 

And yet, even though separations between the statistical performances of different ERMs are known, one of the most fundamental questions in stochastic convex optimization remained unanswered: 
\begin{center}
   {\em what is the worst-case sample complexity of ERMs in SCO?}
\end{center}
A standard covering argument (e.g.~\cite{shalev2014understanding,anthony1999neural}) shows it to be at most
$\tilde O (\frac{d}{\eps^2} )$. \citet{feldman2016generalization} demonstrated it
to be at least
$\Omega (\frac{d}{\eps}+\frac{1}{\eps^2} )$.
The gap between the lower bound
and the upper bound remained open, and was proposed as an open problem in~\cite{feldman2016generalization}.

\subsection{Our contribution}
We resolve the afforementioned open problem and complete the picture by proving that the sample complexity
of ERMs in SCO is actually at most
\[ \tilde O \left( \frac{d}{\eps}+\frac{1}{\eps^2} \right).\]

Notice that this bound separates the statistical complexity of ERMs from uniform convergence. The statistical complexity of uniform convergence is the number of examples required so that typically for \emph{every} model in the parameter space, the empirical loss is close to the true population loss.
In SCO, \citet{feldman2016generalization} provided a lower bound 
that matches the well-known upper bound \cite{anthony1999neural, shalev2014understanding} and demonstrated that
the statistical complexity of
uniform convergence
is $\tilde \Theta (\frac{d}{\eps^2} )$. Our result shows that, while different ERMs may exhibit different performances, all ERMs are better than the worst-case uniform convergence rate. Even without algorithmic assumptions, learning in SCO is easier than uniform convergence (distinctively from other classical models such as PAC learning).

A second notable aspect of our result is that, usually, the term $\tfrac{1}{\eps^2}$ corresponds to noise
or to the agnostic setting of learning. The term $\tfrac{d}{\eps}$, on the other hand,
corresponds to a realizable setting.
It is somewhat surprising that these two different types of ``behaviors'' are simultaneously appearing in a single model.
In stochastic convex optimization, the ERM principle somehow intrinsically combines low dimensional agnostic learning
and high-dimensional realizable learning.

The above result is true for learning bounded convex Lipschitz functions with respect to an $\ell_2$-bounded domain. We  generalize our result
and analyze the statistical complexity of ERMs with respect to general norms. 
For any norm and its corresponding unit ball $\cK$, and for learning bounded convex Lipschitz functions, the statistical complexity of ERMs with error $O(\eps)$ is at most
\[ \tilde O \left( \frac{d}{\eps}+ \mathrm{Rad}^{-1}_{\cK}(\eps)\right),\]
where $\mathrm{Rad}^{-1}_{\cK}(\eps)$ measures the statistical complexity of learning over $\cK$ when we restrict the observed losses to be linear. The first term is the dimension-dependent part of the statistical complexity and is the same for all norms.
The second term, which is the dimension-independent part for the $\ell_2$ norm, is not the classical Rademacher complexity of the full class (``convex functions''), but is the Rademacher complexity
of a smaller class
(``linear functions''). 
The above result yields a tight bound for the $\ell_2$ norm, but also yields tight bounds for general $\ell_p$ norms
(\citet{feldman2016generalization} proved matching lower bounds for $\ell_p$ norms).

\section{Formal Setup and Main Result}\label{sec:setup}
Let us formally state the algorithmic problem we explore.
A one-dimensional illustration appears in \Cref{fig:model}. 
We consider a general norm
$\|\cdot\|$ on $\R^d$
and its unit ball $\cK$.
We also consider a general domain~$Z$, which for concreteness we assume to be finite. Our result hold when $Z$ is infinite, but then one needs to carefully 
deal with measurability issues.
We refer to~\cite{bertsekas1973stochastic} for the relevant definitions and formal setup in which our result hold in full generality.

For each $z \in Z$, we consider a function $f_z(x)$ that is convex and $L$-Lipschits with respect to $x$. 
That is, for every $x,y$,
 \[|f_z(x)-f_z(y)| \leq L\norm{x-y}.\]
We are interested in the problem of optimization over the parameter domain $\cK$. We note, though, that there is no loss of generality in assuming $f_z$ is defined over the whole space $\R^d$, because 
every $L$-Lipschitz convex function over $\cK$ can be extended to an $L$-Lipschitz convex function over all of $\R^d$.

We next assume a distribution $\D$ that is supported on the domain $Z$. We consider the convex function~$F$ defined by
$$F(x) = \E_{z \sim \D}[ f_z(x) ].$$
The function $F$ is often referred to as the \emph{population loss} or \emph{true risk},
and is also $L$-Lipschitz. 
The algorithmic goal is to find a minimizer $x^\star$ of $F$.
The set of minimizers of $F$ in $\cK$ is
$$\argm_\cK F =
\{x^\star \in \cK : \forall  x \in \cK \ F(x^\star) \leq F(x)\}.$$

\begin{figure}
  \includegraphics[width=8cm]{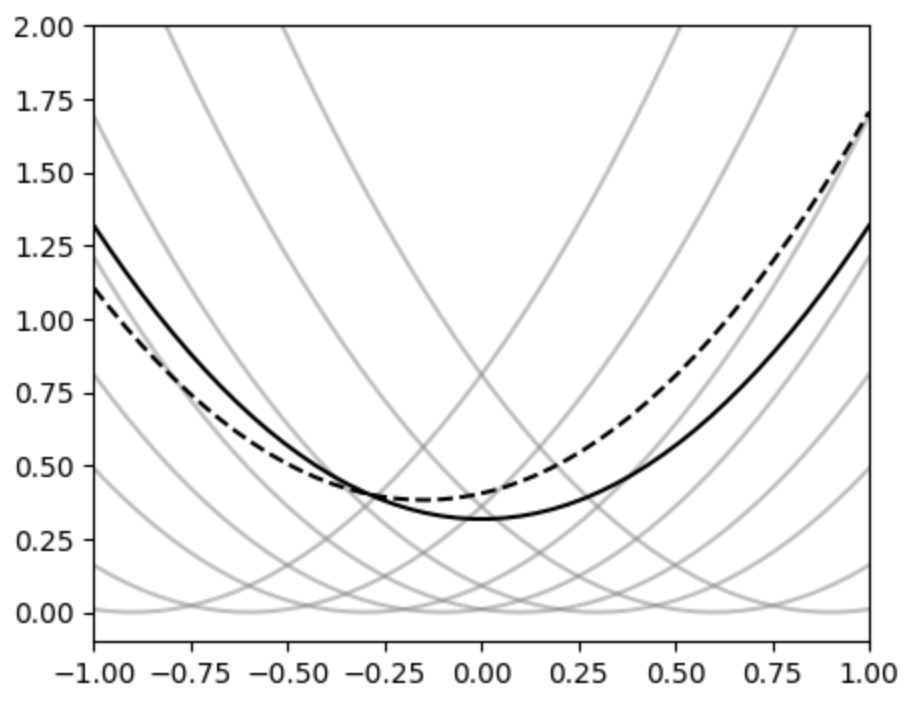}
  \caption{An illustration of a one-dimensional stochastic convex optimization problem.
  The gray lines are the graphs of the different $f_z$'s.
  The solid black line is the target function $F$.
  The minimizer of $F$ is $x^\star=0$.
  The dashed line is an example for a possible $\hat{F}$.
   The minimizer of $\hat{F}$ is $\hat{x} \approx -0.15$.
 }
  \label{fig:model}
\end{figure}
 
 The learning algorithm does know the distribution $\D$. Instead, one assumes that the learner can observe an i.i.d.\ sample $S=(z_1,\ldots, z_n)$ of $n$ examples drawn from the distribution $\D$. A natural strategy in this setting is to follow the ERM principle. Given $S$, we consider the \emph{empirical risk} or \emph{empirical loss}
$$\hat{F} = \frac{1}{n} \sum_{j \in [n]} f_{z_j}.$$
An ERM learning rule is any algorithm that outputs
$$\hat{x} \in \argm_\cK \hat{F}.$$
For an accuracy parameter $\epsilon >0$, the ERM principle is considered successful if 
$$  F(\hat{x}) \leq F(x^\star) + \eps.$$ 
We would like to have the guarantee that 
with as few samples as possible,
any algorithm that follows
the ERM principle is successful.

Our results
provide such a guarantee. 
We first state our result in the cleanest form possible.
Accordingly, the only two parameters 
we focus on are the dimension $d$ and accuracy parameter $\eps$, and we work over the $\ell_2$ norm and its unit ball
$$\cB = \{x \in \R^d : \|x\|_2 \leq 1\}.$$
General norms and other parameters are treated in \cref{sec:main_general} below.

\begin{theorem}\label{thm:main}
Suppose that for every $z\in Z$, the function $f_z$ is  convex, $1$-Lipschitz 
with respect to $\|\cdot\|_2$ and for every $x \in \cB$
we have $|f_z(x)| \leq 1$.
For every distribution $\D$ on $Z$, and for all accuracy parameters $0 < \eps<1$, if
$$n \geq  n_0 = \frac{3d \ln (\frac{40}{\eps})  }{\eps} + \frac{40}{\eps^2}$$
then with probability at least $\frac{3}{4}$
over $S \sim \cD^n$,
for all $x \in \cB$,
$$\hat{F}(x) \leq \hat{F}(x^\star) + \eps \ \
\Longrightarrow \ \
F(x) \leq F(x^\star)+  40\eps.$$
\end{theorem}

The theorem shows that for $n \geq n_0$,
not only that the ERM principle is successful,
but it is also robust in the sense
that it is safe to 
output any point $\hat{x}$ that is merely {\em close} to being a minimizer of~$\hat{F}$.
There is a distinction between ERMs (algorithms
that output a minimizer of $\hat{F}$)
and approximate ERMs (algorithms
that output an approximate minimizer of $\hat{F}$).
This subtle distinction is important
in SCO.
Approximate ERMs
can be more efficient than 
exact ERMs; see~~\cite{shalev2014understanding}.
For example, for $\eps \geq \Omega(1)$,
for all problems, 
sample complexity
of $O(1)$ can be achieved
by a regularized ERM
which is an approximate ERM,
but for some problems,
any exact ERM must have sample complexity
$\Omega(\ln d)$.

\subsection{General norms}\label{sec:main_general}
We next provide a more fine-grained result that generalizes \cref{thm:main} for arbitrary norms and their unit balls. 
We start by introducing the necessary terminology.
Recall that $\norm{\cdot}$ is a norm on $\R^d$ and $\cK \subset \R^d$ is its unit ball.
Denote by $\|\cdot\|_\star$ the dual norm
\[\norm{g}_\star= \sup_{x\in \cK}\ip{g}{x}.\]
The Rademacher complexity of $\cK$ with respect to a sample $S=(g_1,\ldots, g_n) \in (\R^d)^n$ is
\begin{align*}
\rad (\cK, S) = \E_{\sigma}\left[ \norm{\frac{1}{n}\sum_{j=1}^n \sigma_j g_j}_\star\right],\end{align*}
where the expectation is over Rademacher
random variables
$\sigma_1,\ldots,\sigma_n$
(i.e., i.i.d.\ uniform
in $\{\pm 1\}$).
Observe that $\rad(\cK,S)$ corresponds to the standard Rademacher complexity if we think of $\cK$ as a class of linear functions operating on the dual ball;
indeed, for every $S$,
\[    \rad(\cK,S) = \E_{\sigma}\left[\sup_{x \in \cK} \frac{1}{n} \sum_{j=1}^n \sigma_j \ip{g_j}{x}\right].
\]
For an integer $n$,
define
$$\rad(\cK,n)
= \sup_{S} \rad(\cK,S),$$
where the supremum is over $S = (g_1,\ldots,g_n)$ so that $\|g_j\|_\star \leq 1$ for all $j \in [n]$.
For $\eps>0$, define
$$\invrad(\eps) = \min \big\{ n \in \N : \rad(\cK,n)
 < \eps \big\}.$$
It is worth noting that
because
\begin{align*}
\E \left[ \norm{
\frac{1}{n+1}\sum_{j=1}^{n+1} \sigma_j g_j}_\star\right]
& = 
\E \left[ \norm{
\frac{1}{n+1} \sum_{I \subset [n+1]:|I|=n}
\frac{1}{n}\sum_{i \in I} \sigma_j g_j}_\star\right] \\
& \leq 
\frac{1}{n+1} \sum_{I \subset [n+1]:|I|=n}
\E \left[ \norm{
\frac{1}{n}\sum_{i \in I} \sigma_j g_j}_\star\right] ,
\end{align*}
the following monotonicity holds:
$$\rad(\cK,n+1) \leq \rad (\cK,n) .$$
Therefore, if $n \geq \invrad(\eps)$
then $\rad(\cK,n) < \eps$.

The general bound on the sample complexity of ERMs is described in the following theorem.

\begin{theorem}\label{thm:mainG}
Suppose that for every $z\in Z$, the function $f_z$ is  convex, $L$-Lipschitz 
with respect to $\|\cdot\|$ and for every $x \in \cK$
we have $|f_z(x)| \leq c$.
For every distribution $\cD$
on $Z$,
and for all $0 < \eps ,\delta<1$, 
if
$$n \geq  n_0  = \frac{12 c d}{\eps} \ln \Big(\frac{3 L}{\eps} \Big)
+ \invrad\Big( \frac{\eps}{2 L} \Big)+ \frac{8 c^2}{\eps^2} \ln \Big( \frac{4}{\delta} \Big)
$$
then with probability at least $1-\delta$
over $S \sim \cD^n$,
for all $x \in \cK$,
$$\hat{F}(x) \leq \hat{F}(x^\star) + \eps \ \
\Longrightarrow \ \
F(x) \leq F(x^\star)+  40 \eps.$$
\end{theorem}

The sample complexity $n_0$
consists of three different terms.
The first term is of the form 
$\frac{cd}{\eps} \ln (\frac{L}{\eps})$,
and also appears in the $\ell_2$ setting.
In a nutshell, 
the $d  \ln (\frac{L}{\eps})$ part comes from the logarithm of the cover number (see~\cref{lem:covNum} below).
The second term $\invrad\big( \frac{\eps}{L} \big)$ 
correspond to learning the special case of linear functions.
It heavily depends on the norm and its unit ball $\cK$.
For example, when $\frac{\eps}{L} \geq \Omega(1)$,
for the $\ell_2$ norm
it is $O(1)$,
for the $\ell_1$ norm
it is 
$O(\ln d)$, and
for the $\ell_\infty$ norm
it is 
$O(d)$.
The third term $O(\frac{c^2}{\eps^2} \ln  (\frac{1}{\delta}))$ 
is the only place the confidence parameter $\delta$ appears in.

\section{On the proof}
We now outline the proof of \cref{thm:main}.
The proof has two central ingredients.
One ingredient is a first-order optimality condition
for stochastic convex optimization problems
(\cref{prop:struc} in \cref{sec:strut}).
The second ingredient connects properties of ERM
to the Bregman divergence and prove a strong concentration bound
for it.
The mechanism underlying these steps may be useful
for other convex optimization problems.
The proof of the theorem is finally given in \cref{sec:proof}.
The proof of~\cref{thm:mainG} requires a couple of additional ideas,
which are explained in~\cref{sec:general}.

A key property is identified in \cref{clm:prob}.
This claim allows to replace the $\frac{d}{\eps^2}$
term which is ``expected to hold''
by the term $\frac{d}{\eps}$ which is ``correct''.
The intuition for this gain is that
distinguishing between a coin with bias $\frac{1}{2}+\eps$
and a coin with bias $\frac{1}{2}-\eps$ requires~$\frac{1}{\eps^2}$
samples,
but distinguishing between a coin with bias $\eps$
and a coin with bias $2\eps$ can be done with only~$\frac{1}{\eps}$
samples.

Where is the coin?
For the simplicity of the exposition,
we assume that the $f_z$'s are smooth and that the minimizer $x^\star$ is in the interior of $\cB$
(we also included an illustration in \Cref{fig:div}).
In this case, $F$ is also smooth and first-order optimality implies that $\nabla F(x^\star) = 0$.
For each $x \in \cB$, we know that $F(x)-F(x^\star) \geq 0$, and
our goal is to control the empirical difference $\hat{F}(x) - \hat{F}(x^\star)$.
This is achieved, in a nutshell, through the non-negativity of the Bregman divergence.
The Bregman divergence associated with $f_z$ between $x$ and $x^\star$ is
$$D_{f_z}(x,x^\star) := f_z(x) - f_z(x^\star) - \ip{\nabla f_z(x^\star)}{x-x^\star}.$$
It is always non-negative, and it is also bounded from above by four.
The key observation is that
$$F(x)-F(x^\star) = D_F(x,x^\star) = \E_{z \sim \D} [D_{f_z}(x,x^\star)] ,$$
which is explained by
$$\E_{z \sim \D} [\nabla f_z(x^\star)] = \nabla F(x^\star) = 0.$$
We also see that ``$x$ is a bad output''
when $D_F(x,x^\star) \geq 2\eps$.
Instead of directly controlling $\hat{F}(x)-\hat{F}(x^\star)$,
we aim to control the empirical Bregman divergence
$$D_{\hat{F}}(x,x^\star) = \hat{F}(x) - \hat{F}(x^\star) - \ip{\nabla \hat{F}(x^\star)}{x-x^\star}.$$
We can think of $D_F(x,x^\star)$ as the bias
of a coin taking values in~$[0,4]$.
The non-negativity and boundedness of the Bergman divergence allows to prove, via Bernestein inequality, that
if $D_{F}(x,x^\star) \geq 2 \eps$ then
$$\Pr[ D_{\hat{F}}(x,x^\star) \leq \eps]
\leq \exp ( - \Omega(\eps n)).$$
This seems pretty close to our goal;
if ``$x$ is a bad output'' then it is very likely that $D_{\hat{F}}(x,x^\star)$ is large.
Because we actually care about $\hat{F}(x)-\hat{F}(x^\star)$
and not $D_{\hat{F}}(x,x^\star)$,
we also need to control $\nabla \hat{F}(x^\star)$.
This can be achieved via a simple concentration argument because
$$\E \big[\|\nabla \hat{F}(x^\star)\|_2^2\big] \leq \frac{1}{n}.$$
By taking a standard union bound over an $\eps$-net $N$ for $\cB$ of size
$|N| \approx (\tfrac{1}{\eps})^{d}$, we get the final bound
$$\Pr[ \exists x \in \cB \ F(x)-F(x^\star) > 4 \eps , \
\hat{F}(x)-\hat{F}(x^\star) < \eps ] \leq \frac{1}{\eps^2 n} + |N| \exp ( - \Omega(\eps n)).$$
This argument actually works also when the $f_z$'s are {\em not} assumed
to be smooth.
For this, we need to replace ``gradients'' by ``subgradients'';
see \cref{sec:strut} below.
But when $x^\star$ is on the boundary of~$\cB$, even in the smooth case,
we do not know that $\nabla F(x^\star)=0$.
In this case, we additionally need to understand what happens when
$D_{F}(x,x^\star) < 2 \eps$. This is achieved in~\cref{clm:lambdaSmall}.

\begin{figure}
  \includegraphics[width=8cm]{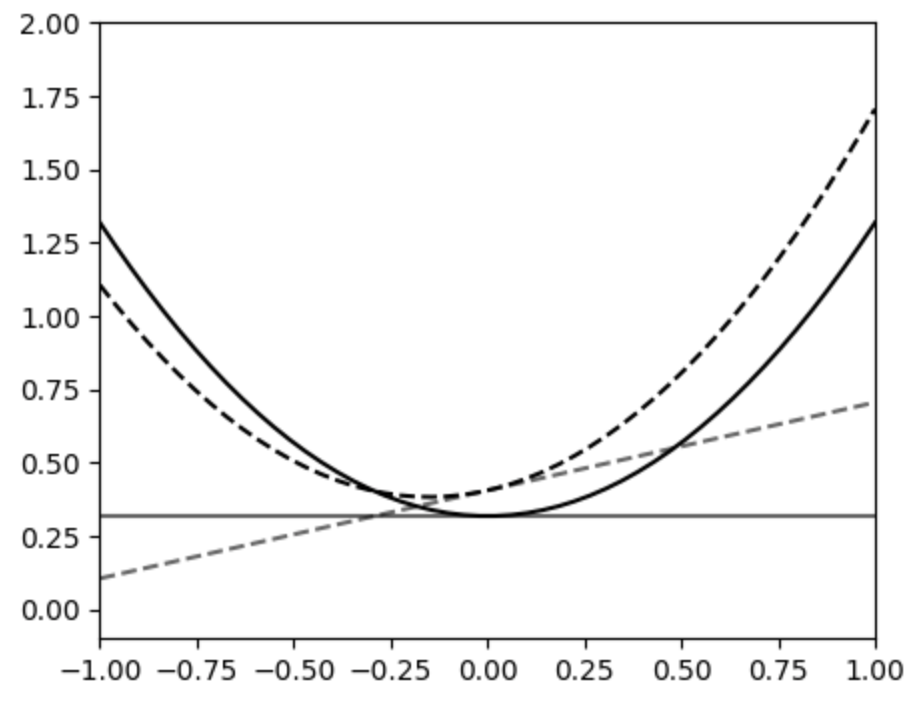}
  \caption{An illustration of the role the Bregman divergence plays in the proof.
  The graphs of $F$ and $\hat{F}$ are from \Cref{fig:model}.
  The solid line is the tangent to $F$ at $x^\star = 0$,
  and the dashed line is the tangent to $\hat{F}$ at $x^\star$.
  The tangents are always below the graphs due to convexity.
  The differences between the graphs and the tangents are the Bregman divergences.
  By controlling the tangent to $\hat{F}$, we can control the behavior of ERM.
  Controlling the tangent to $\hat{F}$ is easier than controlling $\hat{F}$
  because it is an affine function.
 }
  \label{fig:div}
\end{figure}

\section{Preliminaries}
We next develop several preliminary results, and provide the background needed
for the proof.

\subsection{First-order optimality}
\label{sec:strut}
The subgradient of $f :\R^d \to \R$ at $x \in \R^d$ is 
$$\partial f(x) = \{ g \in \R^d : \forall y \in \R^d \ f(y) \geq f(x) + \ip{g}{y-x}\}.$$
If $f$ is convex then the set $\partial f(x)$ is never empty~\cite{rockafellar1997convex}. If $f$ is convex and $L$-Lipschitz with respect to $\|\cdot\|$ then
$$g \in \partial f(x) \ \Longrightarrow \ \|g\|_\star \leq L.$$
A deep property of stochastic functions follows from the subgradient sum property
(because it is a central idea, we included a proof in ~\cref{sec:app}). For a stochastic convex function, the subgradient can be thought as the expectation of the subgradients~\cite{bertsekas1973stochastic}; namely, if $F = \E_{z\sim \D}[f_z],$ then
 for every $g\in \partial F(x)$, 
 for each $z \in Z$,
 there exists $g_z\in \partial f_z(x)$ so that
 $$g= \E_{z\sim \D}[g_z].$$ 
This leads to the following stochastic first-order optimality condition.

\begin{prop}[Stochastic first-order condition]
\label{prop:struc}
Let $\|\cdot\|$ be a norm on $\R^d$ and denote by $\cK$
its unit ball.
Let $\cD$ be a distribution on the finite
set $Z$.
Suppose that for every $z \in Z$, the function $f_z$ is  convex and $L$-Lipschitz. 
Assume that
\[ x^\star \in \argm_\cK F\]
where $F =  \E_{z\sim \D}[f_z]$.
Then, for each $z \in Z$, 
there is $g_z \in \partial f_z(x^\star)$ so that
$\|g_z\|_\star \le L$ and for all $x \in \cK$, 
\begin{equation}\label{eq:foo}
\ip{G}{x-x^\star} \geq 0\end{equation}
where
$G =  \E_{z\sim \D}[g_z] \in \partial F(x^\star)$.
\end{prop}

\begin{proof}
The first-order optimality condition 
(see e.g.~\cite{rockafellar1997convex,bubeck2015convex}) applied to $F$ states that there
is $G \in \partial F(x^\star)$ so that for all $x \in \cK$,
$$\ip{G}{x-x^\star} \geq 0.$$
By the subgradient sum property,
there are $g_z\in \partial f_z(x)$ such that $G=\E[g_z]$. Because $f_z$ is $L$-Lipschitz,
we know that $\|g_z\|_\star \leq L$.
\end{proof}

\subsection{The Bregman divergence}
The Bregman divergence measures
the difference between a convex function
and its first-order approximation. 
It is often defined for smooth functions,
but for our purposes, 
we require a definition that makes sense for genercal SCO problems.
For a stochastic function $F= \E_{z\sim \D}[f_z]$ and a point $x^\star$, the Bregman divergence associated with $F$ at $x^\star$ is
\[D(x,x^\star) = D_F(x,x^\star) = F(x)-F(x^\star) - \ip{G}{x-x^\star},\]
for some $G\in \partial F(x^\star)$. This mapping depends on the choice of $G$, but to avoid cumbersome notations we will suppress this dependence. Importantly, we care about the Bregman divergence at $x^\star$ which is a minimizer of $F$, and we assume that $G \in \partial F(x^\star)$ is chosen to be a sub-gradient whose existence is guaranteed by \cref{prop:struc}.

Given a sample $S = (z_1,\ldots,z_n)$, the empirical Bregman divergence is
\[ \hat{D}(x,x^\star) = \hat F(x)- \hat F(x^\star) -\ip{\hat G}{x-x^\star},\]
where $\hat G= \frac{1}{n}\sum_{i=1}^n g_{z_i} \in \partial \hat{F}(x^\star)$. 
Two key observations are that by the convexity of all the $f_z$'s, both $D(x,x^\star)$ and $\hat{D}(x,x^\star)$ are  non-negative, and
\[ \E_{S\sim \D^n}[\hat {D}(x,x^\star)] = D(x,x^\star).\]

\subsection{Covering numbers}
For $\eps>0$, the $\eps$-cover number $\cov_\cK(\eps)$ of $\cK$ is the minimum 
integer $m$ so that there
is a net $N \subset \cK$
of cardinality $|N| = m$
so that for every $x \in \cK$
there is $y \in N$
so that $\|x-y\| \leq \eps$.
The following is a standard bound on the cover numbers of unit balls of norms. 

\begin{lemma}
\label{lem:covNum}
Let $\|\cdot\|$ be a norm on $\R^d$ and denote by $\cK$
its unit ball.
For every $\eps>0$,
$$\cov_\cK(\eps) \leq \Big(\frac{2(1+\eps)}{\eps}\Big)^d.$$
\end{lemma}

\begin{proof}
Let $Y$ be a subset of $\cK$
of maximum size so that for every $y\neq y'$ in $Y$ we have $\|y-y'\| > \eps$.
It follows that
\begin{align*}
(1+\eps)^d |\cK|
& = |(1+\eps) \cK| \\
& \geq \Big| \bigcup_{y \in Y} \Big(y + \frac{\eps}{2} \cK\Big)\Big| \\
& = |Y| \cdot \Big|\frac{\eps}{2} \cK \Big|
= \Big( \frac{\eps}{2}\Big) ^d \cdot |Y| \cdot |\cK|,     
\end{align*}
where $|\cdot|$ denotes volume.
Because $Y$ is maximal,
for every $x \in \cK$ there is $y \in Y$ so that 
$\|x-y\| \leq 2 \eps$.
\end{proof}

\subsection{Bernstein's inequality}
Finally, we recall the following concentration bound, due to Bernstein:
\begin{lemma*}[Bernstein's inequality]
Let $R_1,\ldots,R_n$ be i.i.d.\ random variables taking values in $[-M,M]$,
each with mean zero and variance $\sigma^2$.
Then, for all $t > 0$, we have
$$\Pr \big[ \sum_{j \in [n]} R_j \geq t \big]
\leq \exp \Big( - \frac{t^2}{2 \sigma^2 n + (2Mt/3)} \Big).$$
\end{lemma*}
\section{The Euclidean norm}
In this section, we prove~\cref{thm:main}.
Let $x^\star$ be an arbitrary minimizer of $F$ in $\cB$. Let $G\in \partial F(x^\star)$ be the subgradient given in \cref{prop:struc}. For a fixed sample, we also denote:
\[ \hat G = \frac{1}{n} \sum_{i=1}^n g_{z_i}.\] 

We start the proof by providing both a lower bound on the Bregman divergence at spurious empirical risk minimizers, as well as an upper bound for the empirical Bregman divergence at empirical risk minimizers. 

The following claim says that for a point to be a spurious risk minimizer, the Bregman divergence must be ``large".

\begin{claim}
\label{clm:lambdaSmall}
For all $x \in \cB$, if
$$\hat{F}(x) - \hat{F}(x^\star)\leq 5\eps$$
and
\begin{align*}
\|\hat{G}-G\|_2 \leq \eps
\end{align*}
then
\begin{align*}
D(x,x^\star) \ge F(x)- F(x^\star) - 7\eps .
\end{align*}
\end{claim}

\begin{proof}
\ignore{
It is helpful to view the Bregman divergence as the difference
between the target function and it first-order approximation by an affine function.
Denote by $\A_z$ the affine function
$$\A_z(x) = f_z(x^\star) + \ip{\g_z}{x-x^\star}$$
and denote by $A$ the affine function
$$A(x) = \E_{z \sim \D} \A_z(x) = F(x^\star) + \ip{G}{x-x^\star}.$$
It follows that
$$D(x,x^\star) = F(x) - A(x).$$
}

Because
\begin{align*}
\ip{G}{x-x^\star}
& = \ip{G - \hat{G}}{x-x^\star}
+ \ip{\hat{G}}{x-x^\star} \\
& \leq 2 \eps 
+ \ip{\hat{G}}{x-x^\star} \\
& = 2 \eps 
+ \hat{F}(x^\star) + \ip{\hat{G}}{x-x^\star}
- \hat{F}(x^\star) \\
&= 2 \eps 
+ \hat{F}(x) -D(x,x^\star)
- \hat{F}(x^\star) \\
& \leq 2 \eps 
+ \hat{F}(x) - \hat{F}(x^\star)
 & \hat{D}(x,x^\star)\geq 0 \\
 & \leq 7 \eps ,
\end{align*}
we can bound
\begin{align*}
D(x,x^\star)
& = F(x) - F(x^\star) - \ip{G}{x-x^\star} \\
& \geq F(x) - F(x^\star) - 7 \eps.
\end{align*}

\ignore{
If we denote by $\hat{A}$ the empirical analog of $A$, then one can see that:
\begin{align}\label{eq:ahata}
A(x) - \hat{A}(x)
& = F(x^\star) -\hat{F}(x^\star) + \ip{G-\hat{G}}{x-x^\star}
 \leq \eps + 2 \eps .
\end{align}
Therefore:
\begin{align*}
F(x)
& = A(x) + D(x,x^\star) \\
& \leq \hat{A}(x) + 3 \eps + D(x,x^\star) &\textrm{\cref{eq:ahata}}\\
& \leq \hat{F}(x) + 3 \eps + D(x,x^\star) &\textrm{convexity of $\hat F$}\\
& \leq \hat{F}(x^\star) + 4 \eps + D(x,x^\star) &\hat F(x)-\hat F(x^\star)\le \epsilon\\
& \leq F(x^\star) + 5 \eps + D(x,x^\star). &\hat F(x^\star)- F(x^\star)\le \epsilon
\end{align*}
The result now follows by rearranging terms.}
\end{proof}

Next, we relate data of the form ``$x$ is an ERM''
to an upper bound on the empirical Bregman divergence.

\begin{claim}
\label{clm:event}
For all $x \in \cB$, if
\begin{align*}
\hat{F}(x) - \hat{F}(x^\star) \leq 5 \eps
\end{align*}
and
\begin{align*}
\| \hat{G}-G\|_2 \leq \eps
\end{align*}
then
\begin{align*}
\hat{D}(x,x^\star)  \leq 7 \eps  .
\end{align*}
\end{claim}

\begin{proof}
\ignore{
\begin{align*}
5 \eps
& \geq \hat{F}(x) - \hat{F}(x^\star) \\
& = \hat{F}(x) - \hat{A}(x) + \hat{A}(x) - \hat{F}(x^\star) \\
& = \hat{D}(x,x^\star)  + \ip{\hat{G}}{x-x^\star}  \\
& = \hat{D}(x,x^\star)  + \ip{\hat{G}-G}{x-x^\star} +\ip{G}{x-x^\star}  \\
& \geq \hat{D}(x,x^\star)  -2 \eps.  & \cref{eq:foo} 
\end{align*}
}
\begin{align*}
\hat{D}(x,x^\star) &= \hat{F}(x) - \hat{F}(x^\star) - \ip{\hat{G}}{x-x^\star}\\
&\leq 5\epsilon - \ip{\hat{G}}{x-x^\star}\\
&=5\epsilon - \ip{\hat{G}-G}{x-x^\star} - \ip{G}{x-x^\star} \\
&\leq 5\epsilon - \ip{\hat{G}-G}{x-x^\star}\\
&\leq 7\epsilon .
 & \cref{eq:foo} 
\end{align*}
\end{proof}
So far, we showed that for a point to be a spurious empirical risk minimizer, the Bregman divergence must be large yet the empirical Bregman divergence is small. The last claim we need is a concentration bound that relates the empirical and true Bregman divergences.

\begin{claim}
\label{clm:prob}
Suppose that for each $z \in Z$, the function $f_z$ is convex, $1$-Lipschitz 
with respect to $\|\cdot\|_2$,
and for all $x \in \cB$
we have $|f_z(x)|\leq 1$.
Then, for every $x\in \cB$,
$$\Pr\Big[\hat{D}(x,x^\star) \leq  \frac{D(x,x^\star)}{2} \Big]  \leq \exp \Big( - \frac{D(x,x^\star)  n}{40} \Big),$$
where the probability is over $z_1,\ldots,z_n$ i.i.d.\
samples from $\cD$.
\end{claim}
\begin{proof}
  The random variable
\begin{align*}
\hat{D}(x,x^\star)
= \frac{1}{n} \sum_j D_{z_j}(x,x^\star)
\end{align*}
is the average of $n$ i.i.d.\ variables $R_1,\ldots,R_n$, each taking values in $[0 ,4]$
with expectation $\lambda:=D(x,x^\star)$. 
By the Bhatia–Davis inequality, the variance of the random variables is at most $$(4-\lambda)\lambda \leq 4 \lambda.$$
By Bernstein's inequality,
\begin{align*}
\Pr\Big[\hat{D}(x,x^\star)  \leq \frac{\lambda}{2} \Big]
& = \Pr\Big[ \sum_j (R_j - \lambda ) \leq - \frac{\lambda n}{2} \Big] \\
& \leq \exp \Big( - \frac{\lambda^2 n^2}{4 ( 8 \lambda n + 2 \lambda n)} \Big) \\
& \leq \exp \Big( - \frac{\lambda  n}{40} \Big) .
\end{align*}
\end{proof}

%

\subsection{Putting it together}
\label{sec:proof}

Assume that $n \geq n_0$.
We start with a simple claim that controls the behavior of the linear part, and we show that:
\begin{equation}\label{clm:linear} \Pr[\|\hat{G}-G\|^2_2 > \eps^2 ] \leq \frac{4}{\eps^2 n}.
\end{equation}
Indeed, 
\begin{align*}
\E \|\hat{G}-G\|^2_2
& = \frac{1}{n^2} \sum_{j,j'} \E \ip{\g_{z_j} -G }{\g_{z_{j'}} -G}  \\
& = \frac{1}{n^2} \sum_{j} \E \|\g_{z_j} -G \|^2 \\
& \leq \frac{4}{n}.
\end{align*}
Now, \cref{clm:linear} follows from Markov's inequality.

Next, denote by $E$ the event that there is
$x \in \cB$ so that
$$\hat{F}(x) - \hat{F}(x^\star)\leq \eps$$
but
$$F(x) - F(x^\star)> 40 \eps.$$
To prove the final result, we need to show that
\begin{equation}\label{eq:thanfive} \Pr(E)\le \frac{1}{4} .\end{equation}
Towards this, denote by $H$ the event that
\begin{align*}
\|\hat{G}-G\|_2 \leq \eps  .
\end{align*}
Next, let $N$ be an $\frac{\eps}{3}$-net for $\cB$ of size at most $(12/\eps)^d$.
Denote by $\hat N$ the set of points $x \in N$ so that
$$\hat{F}(x)- \hat{F}(x^\star) \leq 5 \eps$$
but
$$F(x) -  F(x^\star) > 35 \eps.$$
Because $F$ and $\hat{F}$ are Lipschitz,
the net property implies that if $E$ holds then $\hat N$ is a non-empty set.
By \cref{clm:lambdaSmall},
conditioned on $H$, if $x\in \hat N$ then
$$D(x,x^\star)   \geq 16 \eps.$$
On the other hand, for every $x$, conditioned on $H$, if $x\in \hat{N}$ then by \cref{clm:event} we have
$$\hat{D}(x,x^\star) \leq  8 \eps .$$
Let $E_x$ be the event that $\hat D(x,x^\star)\le 8\eps$. We obtain, therefore:
$$E\cap H  \subseteq \bigcup_{x \in N : D(x,x^\star) > 16 \eps} E_x.$$
By \cref{clm:prob},
if $D(x,x^\star) > 16 \eps$ then
$$\Pr[E_x] \leq \exp \Big( - \frac{\eps  n}{3} \Big).$$
By the union bound, \cref{lem:covNum}, \cref{clm:linear}, and because $n \geq n_0$,
\begin{align*}
\Pr[E]
& \leq \Pr[\neg H] +
|N| \exp \Big( - \frac{\eps  n}{3} \Big) \\
& \leq \frac{4}{\eps^2 n}
+ \Big( \frac{12}{\eps} \Big)^d \exp \Big( - \frac{\eps  n}{3} \Big) \\
&\le \frac{1}{4} .
\end{align*}
This completes the proof of~\cref{thm:main}.

\section{General norms}
\label{sec:general}

In this section, we prove~\cref{thm:mainG}.
Our first task is to handle the Lipschitz constant $L$
and the bound $c$ on the $\ell_\infty$ norm. 
To properly handle general $L$ and $c$, we need to introduce a truncated version of the Bregman divergence.
Second, we address general confidence parameters. This requires a a uniform concentration result for the gradient at
the optimum, combined with an additional argument that bounds a contractive term that may depend on it. 

For the rest of this section,
we fix a norm $\|\cdot\|$
and its unit ball $\cK$.
We also fix an optimal point $x^\star$, and using \cref{prop:struc} we fix subgradients $g_z \in \partial f_z(x^\star)$  
such that the first-order condition \cref{eq:foo} holds for $G= \E_{z\sim \D}[g_z]$.
As before, we denote by $\hat{G}$
the empirical subgradient
$$\hat{G} = \tfrac{1}{n} \sum_{j=1}^n g_{z_j}.$$

\subsection{The Lipschitz and $\ell_\infty$ constants}
Special care is required
when $c \ll L$.
The reason is that the Bregman divergence can be as large as $L$ regardless of $c$.
We, therefore, introduce a truncated version of the Bregman divergence.
The truncated divergence associated with $f : \cK \to [-c,c]$ is
$$T_f(x,x^\star) = f(x) - f(x^\star) - \max\{-2c,\ip{g}{x-x^\star}\}.$$
The truncated divergence, as the Bregman divergence, is always non-negative. But, in contrast to the Bregman divergence, it is always at most $4c$ because $f$ is bounded by $c$.
Similarly to the Bregman divergence, we use the following notation:
$$T_z = T_{f_z}  , \ T = \E_{z \sim \D} T_z \ \ \text{and} \ \ \hat{T} = \frac{1}{n} \sum_{j=1}^n T_{z_j}.$$
The random variable
$\hat{T}(x,x^\star)$
is the average of $n$ i.i.d.\ variables 
taking values in $[0 ,4c]$
with expectation~$T(x,x^\star)$.
Analogously to \cref{clm:prob},
we can deduce that:

\begin{claim}
\label{clm:probG}
For every $x \in \cK$,
$$\Pr\Big[\hat{T}(x,x^\star) \leq  \frac{T(x,x^\star)}{2} \Big]  \leq \exp \Big( - \frac{T(x,x^\star)  n}{40 c} \Big) .$$
\end{claim}

We now want to develop analgoue statements to \cref{clm:lambdaSmall} and \cref{clm:event}. For that it will be convenient to add the following notation
\begin{equation}\label{eq:lz} \ell_g(a) = \max\{-2c, a- \ip{g}{x^\star}\} .
\end{equation}
The function $\ell_g: \R \to \R$
is always convex and $1$-Lipschitz.
Denote by $\cL$
and $\hat{\cL}$ the following
functions:
\[ \cL(x)= \E_{z\sim \D}[\ell_{g_z}(\ip{g_z}{x})] \quad\mathrm{ and}\quad \hat \cL(x)= \frac{1}{n} \sum_{j=1}^n \ell_{g_{z_j}}(\ip{g_{z_j}}{x}).\]
Both $\cL$ and $\hat \cL$ are convex. 
For a sample $S = (z_1,\ldots,z_n)$,
the {\em representativeness}\footnote{That is the terminology from the book~\cite{shalev2014understanding}.} of $S$ is defined to be
$$\rep(S) = \sup_{x\in \cK} (\cL(x)- \hat \cL(x) ) . $$

\begin{claim}
\label{clm:LsmallG}
Suppose that
$\rep(S)\leq 2\eps$.
Then, for every $x\in \cK$, if
\begin{align*}
\hat{F}(x) - \hat{F}(x^\star) \leq 6 \eps
\end{align*}
then
$$\hat{T}(x,x^\star)  \leq 8 \eps.$$
\end{claim}

\begin{proof}
\begin{align*}
\hat T(x,x^\star) &= \hat{F}(x) - \hat{F}(x^\star) -\hat \cL(x)\\
& \le 6 \eps +\cL(x)-\hat \cL(x) -\cL(x) & \rep(S)\leq 2\eps \\
& \leq 8 \eps - \E_{z\sim \D} \max\{-2c, \ip{g_{z}}{x-x^\star}\} \\
& \leq 8 \eps -  \max\{-2c, \E_{z\sim \D}\ip{g_{z}}{x-x^\star}\}  \\
& \leq 8 \eps -  \ip{G}{x-x^\star} &\cref{eq:foo} \\
& \leq 8 \eps .  
\end{align*}
\end{proof}

\begin{claim}\label{clm:eventG}
Suppose that 
$\rep(S)\leq 2\eps$.
For all $x \in \cK$ if
\begin{equation*}
\hat{F}(x) - \hat{F}(x^\star)\leq 5 \eps
\end{equation*}
then
\begin{align*}
T(x,x^\star) \ge F(x)- F(x^\star) -7\eps .
\end{align*}
\end{claim}
\begin{proof}
Because
\begin{align*}
\cL(x)
& = \cL(x) - \hat{\cL}(x) + \hat{\cL}(x) & \rep(S)\leq 2\eps\\
& \leq 2 \eps + \hat{\cL}(x) \\
& = 2 \eps + \hat{F}(x^\star) + \hat{\cL}(x) -  \hat{F}(x^\star)\\
& \leq 2 \eps + \hat{F}(x)-\hat{T}(x,x^\star) -  \hat{F}(x^\star) \\
& \leq 2 \eps + \hat{F}(x) -  \hat{F}(x^\star) & \hat{T}(x,x^\star) \geq 0 \\
& \leq 7 \eps ,
\end{align*}
we can bound
\begin{align*}
T(x,x^\star)
& = F(x) - F(x^\star) - \cL(x) \\
& \geq F(x) - F(x^\star) - 7 \eps.
\end{align*}

\ignore{

\begin{align*}
F(x)
& = F(x^\star) +\cL(x) + T(x,x^\star) \\
& \leq \hat F(x^\star) + \hat \cL(x) + 3 \eps + T(x,x^\star) &\cref{eq:assspurious} \\
& = \hat{F}(x) -\hat T(x) + 3 \eps + T(x,x^\star) \\
& \le \hat{F}(x)  + 3 \eps + T(x,x^\star) & \hat T(x,x^\star)\ge 0 \\
& \leq \hat{F}(x^\star) + 4 \eps + T(x,x^\star) &\hat F(x)-\hat F(x^\star)\le \epsilon\\
& \leq F(x^\star) + 5 \eps + T(x,x^\star). &\hat F(x^\star)- F(x^\star)\le \epsilon 
\end{align*}}
\end{proof}

\subsection{The linear part}
\label{sec:concBall}

In this section,
we explain how to control 
$\rep(S)$.
We are going to use standard results on 
the Rademacher complexity,
and for that we use the book~\cite{shalev2014understanding}.
In the book, the Rademacher complexity of a set $A \subset \R^n$ is defined to be (and denoted by $R$):
$$\rad(A) = \E_{\sigma}
\left[ \sup_{a \in A} \frac{1}{n} \sum_{j=1}^n \sigma_j a_j \right].$$
For a fixed sample $S = (z_1,\ldots,z_n)$,
we consider the set
$$A = A_S = \{ 
(\ip{g_{z_1}}{x},\ldots,\ip{g_{z_n}}{x}) \in \R^n : x \in \cK\}$$
so that
$$\rad(A) = \rad(\cK,S).$$
We also use the $1$-Lipschitz maps $\ell_{g }: \R \to \R$ that are defined in~\cref{eq:lz}. 
We denote by $\ell = \ell_S$ the
sequence $\ell = (\ell_{g_{z_1}},\ldots,\ell_{g_{z_n}})$, 
and consider the set
$$\ell \circ A
= \{ (\ell_{g_{z_1}}(\ip{g_{z_1}}{x}),\ldots,\ell_{g_{z_n}}(\ip{g_{z_n}}{x})) \in \R^n : x \in \cK\}.$$
It follows from Rademacher calculus \cite[see Lemma 26.9 (contraction lemma)]{shalev2014understanding} that
\begin{align}
\label{eqn:radL}
\rad(\ell\circ A,S) \le  \rad(A) = \rad(\cK,S).    
\end{align}
We also rely on the McDiarmid’s inequality following (see e.g.~\cite[Lemma 26.4]{shalev2014understanding}).

\begin{lemma}[McDiarmid’s inequality]
Let $h : Z^n \to \R$  be so that for some $c_0>0$, for all $j \in [n]$
and for all $z_1,\ldots,z_n,z'_i \in Z$,
\begin{align}
\label{eqn:mcdia}
| h(z_1,\ldots,z_n) - h(z_1,\ldots,z_{i-1},z'_i,z_{i+1},\ldots,z_n)| \leq c_0.
\end{align}
Let $\cD$ be a distribution on $Z$.
Then, for all $\delta >0$,
$$\Pr_{S \sim \cD^n}\left[|h(S) - \mu| 
> c_0 \sqrt{\frac{\ln(2/\delta) n}{2}} \right]
\leq \delta$$
where $\mu = \E_{S \sim \cD^n} [h(S)]$.
\end{lemma}

\ignore{
\begin{theorem}
Suppose that $\nu$ is a distribution over vectors $g$ so that 
almost surely, $|\ell_g(x)|\le c$
for all $x \in \cK$.
Then, with probability at least $1-\delta$, for all $x\in \cK$,
    \[ \cL(x) -\hat \cL(x) \le 2 \E_{S\sim \nu^n} \rad(\ell\circ \cK,S) + c\sqrt{\frac{2\ln \frac{2}{\delta}}{n}}.\]
\end{theorem}
}

We obtain the following important corollary
(see also~\cite[Theorem 26.5]{shalev2014understanding}).
\begin{corollary}\label{cor:radbound}
Suppose that for each $z \in Z$, the function $f_z$ is convex and $L$-Lipschitz with respect to $\|\cdot\|$
and for every $x \in \cK$,
we have $|f_z(x)| \leq c$.
    For every distribution $\D$ over $Z$ and $\delta \ge 0$,
   with probability at least $1-\delta$
    over $S \sim \cD^n$, 
\[ \rep(S)  \leq 2 L \cdot \rad(\cK, n) +  c\sqrt{\frac{2\ln (2/\delta)}{n}}.\]
\end{corollary}
\begin{proof}
For each $z \in Z$,
we have 
$|\ell_{g_z}(x)| \leq 4c$
for all $x \in \cK$.
The map
$S \mapsto \rep(S)$ satisfies~\cref{eqn:mcdia} with $c_0 = \frac{2c}{n}$.
By~\cref{eqn:radL} and~\cite[Lemma 26.2]{shalev2014understanding}, 
because $\|g_z\|_\star \leq L$
for all $z \in Z$,
\begin{align*}
\E_{S \sim \cD^n}[ \rep(S)]
\leq 2 \E_{S \sim \cD^n}[ \rad(\ell_S \circ A_S)]
\leq 2 L \cdot \rad(\cK, n).
\end{align*}
By McDiarmid's inequality, therefore, with probability 
at least $1-\delta$,
$$\rep(S)
\leq 2 L \cdot \rad(\cK, n) + \frac{2c}{n}  \sqrt{\frac{\ln(2/\delta) n}{2}} .$$
\end{proof}

\subsection{Putting it together}

The proof from here is similar to the simpler setup in \cref{thm:main}. 
Assume that $n \geq n_0$, as in the theorem statement.
Next, denote by $E$ the event that there is
$x \in \cK$ so that
$$\hat{F}(x) - \hat{F}(x^\star)\leq \eps$$
but
$$F(x) - F(x^\star)> 40 \eps.$$
The final result then follows if we prove that
\begin{equation}\label{eq:thandelta} \Pr(E)\le \delta .\end{equation}
Towards this, denote by $H$ the event that
\begin{align*}
\rep(S) \leq 2\eps  \ \
\end{align*}
By \cref{cor:radbound},
$$\Pr[\neg H] \leq \frac{\delta}{2} .$$
Next, let $N$ be an $\frac{\eps}{3L}$-net for $\cK$ of minimal size.
Denote by $\hat N$ the set of points $x \in N$ so that
$$\hat{F}(x)- \hat{F}(x^\star) \leq 5 \eps$$
but
$$F(x) -  F(x^\star) > 35 \eps.$$
Because $F$ and $\hat{F}$ are Lipschitz,
the net property implies that if $E$ holds then $\hat N$ is a non-empty set.
By \cref{clm:eventG},
conditioned on $H$, if $x\in \hat N$ then
$$T(x,x^\star)   \geq 16 \eps.$$
On the other hand, by \cref{clm:LsmallG}, for every $x \in \cK$, conditioned on $H$, if $x\in \hat{N}$ then 
$$\hat{T}(x,x^\star) \leq  8 \eps .$$
Let $E_x$ be the event that $\hat T(x,x^\star)\le 8\eps$. We obtain, then:
$$E\cap H  \subseteq \bigcup_{x \in N : T(x,x^\star) > 16 \eps} E_x.$$
By \cref{clm:probG},
if $T(x,x^\star) > 16 \eps$ then
$$\Pr[E_x] \leq \exp \Big( - \frac{\eps  n}{3c} \Big).$$
By the union bound and~\cref{lem:covNum},
\begin{align*}
\Pr[E]
& \leq \Pr[\neg H] +
|N| \exp \Big( - \frac{\eps  n}{3c} \Big) \\
& \leq \frac{\delta}{2}
+ \Big( \frac{12L}{\eps} \Big)^d \exp \Big( - \frac{\eps  n}{3c} \Big) .
\end{align*}
The proof of~\cref{thm:mainG} is now complete 
because
\begin{align*}
\frac{3c^2\ln(4/\delta)}{\epsilon^2}
\ge 
\frac{3c\ln(2/\delta)}{\epsilon}
;\end{align*}
the last inequality is true for $c \geq \eps$, which we can assume because otherwise the result trivially holds.

\bibliographystyle{plainnat}
\bibliography{main}

\appendix

\section{The subgradient sum theorem}
\label{sec:app}

For completeness, we provide a proof of the subgradient sum theorem
(see Rockafellar's book~\cite{rockafellar1997convex}).
We added a couple of figures that illustrate the main idea
(\Cref{fig:f1f2} and~\Cref{fig:z1z2}).
We work in the more general setting when the function is defined on a convex subset of $\R^d$.
For $f : \cK \to \R$ where 
$\cK \subseteq \R^d$ is convex,
define
$$\partial_\cK f(x) = \{ g \in \R^d : \forall y \in \cK \ f(y) \geq f(x) + \ip{g}{y-x}\}.$$

\begin{theorem*}[subgradient sum]
\label{thm:diffSum}
Let $\cK \subseteq \R^d$
be a convex set with non-empty interior.
If $f_1: \R^d \to \R$ and $f_2 : \cK \to \R$ are convex and Lipschitz,
then for all $x \in \cK$,
$$\partial_\cK (f_1+f_2)(x) = \partial_{\R^d} f_1(x) + \partial_\cK f_2(x).$$
\end{theorem*}

This generality helps, for example, to deduce the 
first-order optimality condition we rely on.

\begin{claim*}[first-order optimality]
Let $f : \R^d \to \R$ be Lipschitz and convex.
Let $\cK$ be convex and closed, and let $x_0 \in \argm_\cK f$.
Then, there is $g \in \partial_{\R^d} f(x_0)$ so that
$\ip{g}{x-x_0} \geq 0$ for all $x \in \cK$.
\end{claim*}

\begin{proof}
The point $x_0$ is also a minimizer of $f+1$ as a function on $\cK$.
Therefore, $0 \in \partial_\cK (f+1)(x_0)$.
By the subgradient sum theorem, there is $g \in \partial_{\R^d} f(x_0)$
so that $-g \in \partial_\cK 1(x_0)$.  The latter says that for all $x \in \cK$,
we have $1 \geq 1 + \ip{-g}{x-x_0}$.
\end{proof}

\begin{proof}[Proof of the subgradient sum theorem]
Let $g \in \partial_\cK (f_1+f_2)(x)$.
Define the convex Lipschitz functions $h_1,h_2$:
$$h_1(y) = f_1(y) - f_1(x) - \ip{g}{y-x}$$
and
$$h_2(y) = f_2(y) - f_2(x).$$
So,
$x$ is a minimizer of $h_1+h_2$ in $\cB$,
and $h_1(x)=h_2(x) = 0$.
Define the convex sets
\begin{align*}
Z_1 = \{(y_1,\lambda_1) : y_1 \in \R^d   , \lambda_1 \geq h_1(y_1)\}
\subset \R^d \times \R
\end{align*}
and
\begin{align*}
Z_2 = \{(y_2,\lambda_2) : y_2 \in \cK  , \lambda_2 \leq - h_2(y_2)\}
\subset \R^d \times \R .
\end{align*}
The interiors of these convex sets are disjoint because the minimum of $h_1+h_2$ is zero.
It follows that there is a hyperplane separating them.
This hyperplane is not ``vertical''
because the projection of both to $\R^d$ contains $\cK$;
that is, the separating hyperplane is a graph of an affine function on $\R^d$.
The point $(x,0)$ belongs to both sets,
so the hyperplane must contain it.
Overall, there is $v \in \R^d$ so that
$$\forall (y_1,\lambda_1) \in Z_1 \ \ \ \ \ip{v}{y_1-x} \leq \lambda_1$$
and
$$\forall (y_2,\lambda_2) \in Z_2 \ \ \ \ \ip{v}{y_2-x} \geq \lambda_2.$$
In other words, for all $y \in \R^d$,
$$\ip{v}{y-x}  \leq h_1(y) $$
and for all $y \in \cK$,
$$\ip{v}{y-x} \geq - h_2(y).$$
Going back to $f_1,f_2$, we get that for all $y \in \R^d$ ,
$$\ip{v}{y-x}  \leq
f_1(y) - f_1(x) - \ip{g}{y-x}$$
and for all $y \in \cK$
$$\ip{-v}{y-x}  \leq f_2(y) - f_2(x).$$
The first inequality says that $g+v \in \partial_{\R^d} f_1(x)$
and the second that $-v \in \partial_\cK f_2(x)$.
\end{proof}

\begin{figure}
  \includegraphics[width=8cm]{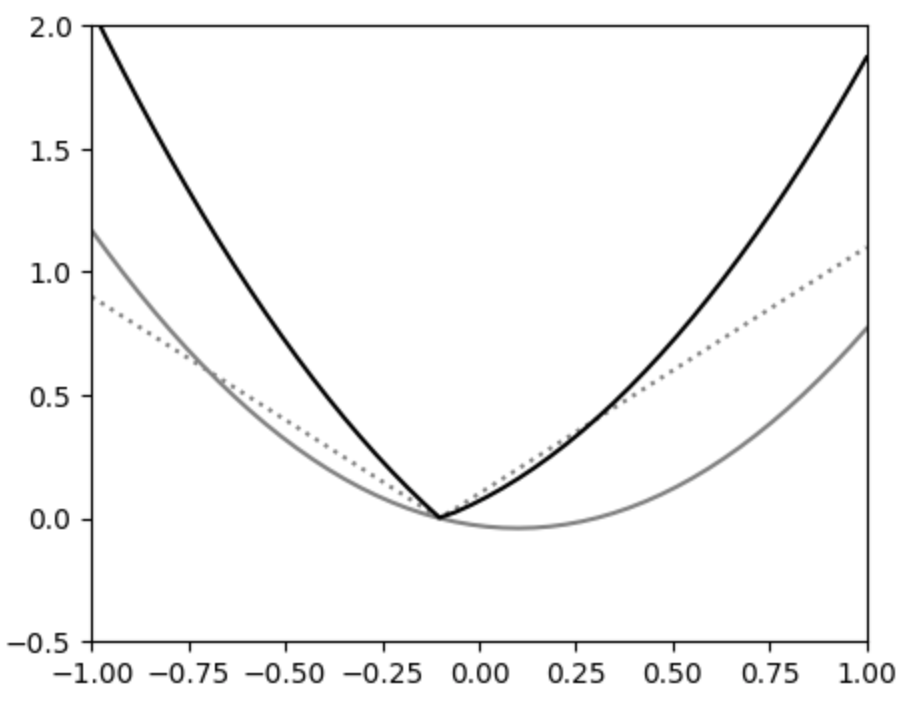}
  \caption{Part I of the illustration of the proof of the subgradient sum theorem.
  The function $f_1(x) = (x-0.1)^2 - \tfrac{1}{25}$ is the gray line.
  The function $f_2(x) = |x+0.1|$ is the dashed line.
  Their sum $f_1+f_2$ is the solid black line.
  The point of interest is $x = -\tfrac{1}{10}$,
  which is the minimum of $f_1+f_2$. \\ \
 }
  \label{fig:f1f2}
  \includegraphics[width=8cm]{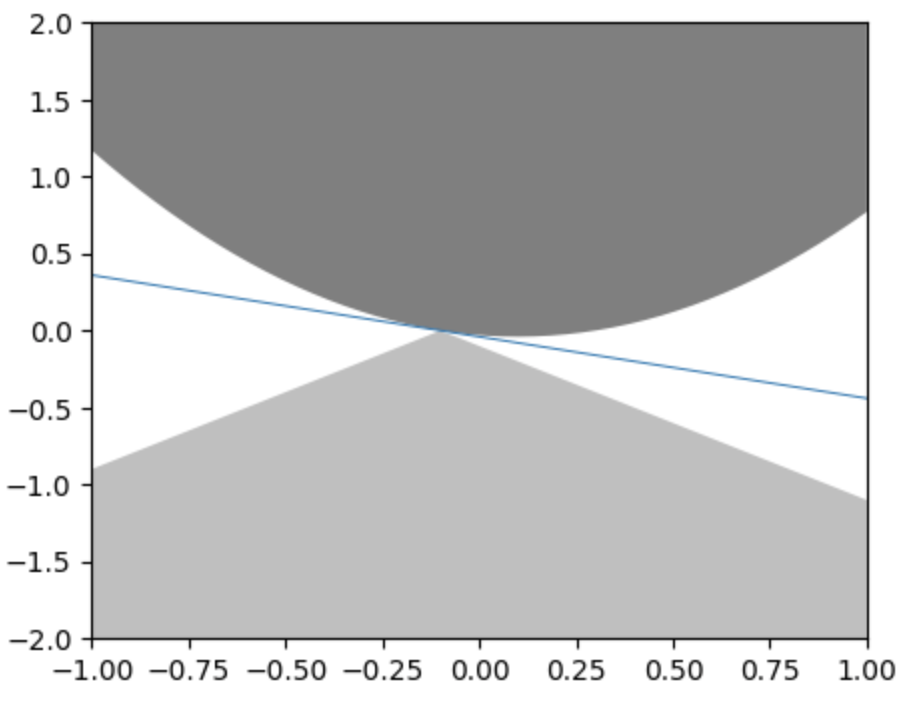}
  \caption{Part II of the illustration of the proof of the subgradient sum theorem.
The convex set $Z_1$ is the epigraph of $f_1$.
The convex set $Z_2$ is ``minus'' the epigraph of $f_2$.
The point of interest is $x = -\tfrac{1}{10}$,
and $0 \in \partial_\cB (f_1+f_2)(x)$.
The two convex sets meet at $(x,0)$.
There is a line separating $Z_1$ and $Z_2$ going through
$(x,0)$. This line shows that we can write $0$ as
a subgradient of $f_1$ plus a subgradient of $f_2$.
 }
  \label{fig:z1z2}
\end{figure}

\end{document}